\documentclass{article}

\newcommand{\name}{Playgol}
\newcommand{\nopi}{\name{}$_{NH3}$}

\usepackage{ijcai19}

\usepackage{times}
\usepackage{soul}
\usepackage{url}
\usepackage[hidelinks]{hyperref}
\usepackage[utf8]{inputenc}
\usepackage[small]{caption}
\usepackage{graphicx}
\usepackage{amsmath}
\usepackage{booktabs}
\usepackage{algorithm}
\usepackage{microtype}
\usepackage{amsthm}
\usepackage{pgfplots}
\usepackage{tikz}
\usepackage{subcaption}
\usepackage{inconsolata}
\usepackage[frozencache=true]{minted}
\usepackage{algpseudocode}
\usepackage{listings}
\theoremstyle{definition}

\newcommand{\tw}[1]{\texttt{#1}}
\newcommand{\M}[2]{\ensuremath{\mathcal{M}^{#1}_{#2}}}

\newtheorem{definition}{Definition}

\newtheorem{theorem}{Theorem}

\newtheorem{proposition}{Proposition}

\usepackage{enumitem}

\newenvironment{myitemize}{
\begin{itemize}[leftmargin=4mm]
  \setlength{\itemsep}{1pt}
  \setlength{\parskip}{0pt}
  \setlength{\parsep}{0pt}
}{\end{itemize}}

\title{Playgol: learning programs through play}
\author{
 Andrew Cropper
 \affiliations
 University of Oxford
 \emails
 andrew.cropper@cs.ox.ac.uk
}

\lstnewenvironment{myalgorithm}[1][] 
{
    \lstset{ 
        mathescape=true,
        numbers=left,
        escapeinside={*}{*},
        numberstyle=\footnotesize\ttfamily,
        basicstyle=\footnotesize\ttfamily,
        keywordstyle=\footnotesize\color{black}\bfseries\em,
        keywords={,input, output, return, datatype, function, func, in, if, else, for, foreach, while, begin, end, }
        numbers=left,
        xleftmargin=.04\textwidth,
        #1 
    }
}
{}

\begin{document}

\maketitle

\begin{abstract}
Children learn though play.
We introduce the analogous idea of \emph{learning programs through play}.
In this approach, a program induction system (the learner) is given a set of user-supplied \emph{build} tasks and initial background knowledge (BK).
Before solving the build tasks, the learner enters an unsupervised \emph{playing} stage where it creates its own \emph{play} tasks to solve, tries to solve them, and saves any solutions (programs) to the BK.
After the playing stage is finished, the learner enters the supervised \emph{building} stage where it tries to solve the build tasks and can reuse solutions learnt whilst playing.
The idea is that playing allows the learner to discover reusable general programs on its own which can then help solve the build tasks.
We claim that playing can improve learning performance.
We show that playing can reduce the textual complexity of target concepts which in turn reduces the sample complexity of a learner.
We implement our idea in \emph{Playgol}, a new inductive logic programming system.
We experimentally test our claim on two domains: robot planning and real-world string transformations.
Our experimental results suggest that playing can substantially improve learning performance.
We think that the idea of playing is an important contribution to the problem of developing program induction approaches that self-discover BK.
\end{abstract}

\section{Introduction}

Children learn though play \cite{schulz2007preschool,sim2017learning,DBLP:conf/cogsci/SimMX17}.
We introduce the analogous idea of \emph{learning programs through play}.
In this approach, a program induction system (the learner) is given a set of user-supplied \emph{build} tasks and initial background knowledge (BK).
Whereas a standard program induction system would immediately try to solve the build tasks, in our approach the learner first enters an unsupervised \emph{playing} stage.
In this stage the learner creates its own \emph{play} tasks to solve, tries to solve them, and saves any solutions (programs) to the BK.
After the playing stage is finished, the learner enters the supervised \emph{building} stage where it tries to solve the user-supplied build tasks and can reuse solutions learned whilst playing.
The idea is that playing allows the learner to discover reusable general programs on its own which can then be reused in the building stage, and thus improve performance.
For instance, if trying to learn sorting algorithms, a learner could discover the concepts of \emph{partition} and \emph{merge} when playing which could then help learn \emph{quicksort} and \emph{mergesort} respectively.

To further illustrate our play idea, imagine a child that had never seen Lego before.
Suppose you presented the child with Lego bricks and immediately asked them to build a (miniature) house with a pitched roof.
The child would probably struggle to build the house without first knowing how to build a stable wall or how to build a pitched roof.
Now suppose that before you asked the child to build the house, you first left them alone to play with the Lego.
Whilst playing the child may build animals, gardens, ships, or many other seemingly irrelevant things.
However, the child is likely to discover reusable and general concepts, such as the concept of a stable wall.
As we discuss in Section \ref{sec:play}, the cognitive science literature shows that children can better learn complex rules after a period of play rather than solely through observation \cite{schulz2007preschool,sim2017learning,DBLP:conf/cogsci/SimMX17}.
In this paper, we explore whether a program induction system can similarly better learn programs after a period of play.

Our idea of using play to discover useful BK contrasts with most forms of program induction which usually require predefined, often human-engineered, static BK as input \cite{mugg:metagold,crop:metafunc,law:ilasp,inspire,crop:metaopt,gulwani:flashfill,evans:dilp,ellis:latex}.
Our idea is related to program induction approaches that perform \emph{multitask} or \emph{meta} learning \cite{mugg:metabias,dechter:ec,ellis:scc,ellis:ijcai17}.
In these approaches, a learner acquires BK in a \emph{supervised} manner by solving sets of user-provided tasks, each time saving solutions to the BK, which can then be reused to solve other tasks.
In contrast to these supervised approaches, our play approach discovers useful BK in an \emph{unsupervised} manner whilst playing.
Playing can therefore be seen as an unsupervised technique for a learner to discover the BK necessary to solve complex tasks, i.e. a form of unsupervised bootstrapping for supervised program induction.

We claim that playing can improve learning performance.
To support this claim, we make the following contributions:
\begin{myitemize}
 \item We introduce the idea of learning programs through play and show that playing can reduce the textual complexity of target concepts which in turn reduces the sample complexity of a learner (Section \ref{sec:framework}).
 \item We implement our idea in \emph{\name{}}, a new inductive logic programming (ILP) system based on meta-interpretive learning (MIL) \cite{mugg:metagold,crop:metafunc} (Section \ref{sec:playgol}).
 \item We experimentally show on two domains (robot planning and real-world string transformations) that playing can significantly improve learning performance (Section \ref{sec:experiments}).
\end{myitemize}

\section{Related work}
\label{sec:related}

\paragraph{Program induction}
Program induction approaches learn computer programs from data.
Much recent work has focused on task-specific approaches for real-world problems, such as string transformations \cite{gulwani:flashfill}.
By contrast, we are interested in general program induction approaches that work on multiple domains.
Specifically, we want to develop program induction techniques that discover reusable general concepts, which was the goal of many early AI systems, such as Lenant's AM system \cite{lenat:am}.

\paragraph{Meta-program induction}
Program induction approaches use BK as a form of inductive bias \cite{mitchell:mlbook} to restrict the hypothesis space.
Most approaches \cite{mugg:metagold,crop:metafunc,law:ilasp,inspire,crop:metaopt,gulwani:flashfill,evans:dilp,ellis:latex} require as input a fixed, often hand-engineered, BK.
To overcome this limitation, several approaches attempt to acquire BK over time \cite{mugg:metabias,dechter:ec,ellis:scc,ellis:ijcai17}, which can be seen as a form of \emph{meta-learning} \cite{thrun:ltl}.
In ILP, meta-learning, also known as \emph{automatic bias-revision} \cite{structenyears}, involves saving learned programs to the BK so that they can be reused to help learn programs for unsolved tasks.
Curriculum learning \cite{cl} is a similar idea but requires an ordering over the given tasks.
By contrast, our approach, and the aforementioned approaches, do not require an ordering over the tasks.

Lin et al. \cite{mugg:metabias} use a technique called \emph{dependent learning} to allow the MIL system Metagol \cite{metagol} to learn string transformations programs over time.
Their approach uses predicate invention to reform the bias of the learner where after a solution is learned not only is the target predicate added to the BK but also its constituent invented predicates.
The authors show that their dependent learning approach performs substantially better than an independent (single-task) approach.
Dechter et al. \cite{dechter:ec} studied a similar approach for learning functional programs.

These existing approaches perform \emph{supervised} meta-learning, i.e. they \emph{need} a corpus of user-supplied training tasks.
By contrast, our playing approach is \emph{unsupervised}, where the tasks come not from the user but from the system itself.
In other words, our approach allows a learner to discover highly reusable concepts without a user-supplied corpus of training tasks, which Ellis et al. \cite{ellis:scc} argue is essential for program induction to become a standard part of the AI toolkit.


\paragraph{Pre-training and latent structures}

Our playing idea can be seen as an unsupervised pre-training step.
DeepCoder \cite{deepcoder} also has a pre-training step.
In this step, Deepcoder enumerates every hypothesis in the hypothesis space (up to a depth limit) and generates random input/outputs for each hypothesis.
Deepcoder uses the examples and the hypotheses to train a neural network to model the distribution over the user-supplied functions in the BK.
By contrast, \name{} randomly samples play tasks from the instance space and tries to learn the solutions for them.
In addition, whereas DeepCoder learns a distribution over a fixed-set of user-supplied functions, Playgol discovers new programs through play and predicate invention, which can be reused during the building stage.


Our playing idea can also be seen as an unsupervised approach to discover latent features (i.e. predicates).
Dumancic and Blockeel \cite{curled} also consider using unsupervised pre-training to improve the performance of an ILP system (TILDE \cite{tilde}).
Their CUR$^2$LED approach focuses on learning relational latent representations in an unsupervised manner, and uses clustering to obtain latent features.
They show that their approach improves the predicate accuracy of TILDE and reduces the complexity of a learned model (compared to no pre-training).
Although \name{} differs from CUR$^2$LED in many ways, both share the goal of discovering new language constructs in an unsupervised manner.

\paragraph{Playing}
\label{sec:play}
Several studies have shown that children learn successfully when they have the opportunity to choose what they want to do.
Schulz et al. \cite{schulz2007preschool} found that children were able to use self-generated evidence to learn about a causal systems.
Sim and Xu \cite{sim2017learning} found that three-year-olds were capable of forming higher-order generalisations about a causal system after a short play period.
Sim et al. \cite{DBLP:conf/cogsci/SimMX17} showed that children perform significantly better when learning complex clausal rules through free play or by first engaging in free play and then observing, as opposed to solely through observation.
As far as we are aware, there is no research studying whether playing can improve machine learning performance, especially in program induction.

\paragraph{Meta-interpretive learning}
Our idea of learning programs through play is sufficiently general to work with any form of program induction, such as inducing functional programs.
However, to clearly explain our theoretical and empirical results, we formalise the problem in an ILP setting using MIL.
We use MIL for two key reasons.
First, MIL supports learning recursive programs, which is important in the string transformation experiments.
Second, MIL uses predicate invention to decompose problems into smaller problems which can then be reused \cite{crop:metafunc}.

\section{Problem setting}
\label{sec:framework}

We now describe the \emph{learning programs through play} problem, which, for conciseness, we refer to as the \emph{Playgol} problem.

\subsection{Problem definition}

Given a set of tasks and BK, our problem is to induce a set of programs to solve each task.
We formalise the problem in an ILP learning from entailment setting \cite{luc:book}.
We define the input to the problem:

\begin{definition}[\textbf{\name{} input}]
\label{def:playgol-input}
A \name{} input is a tuple $(\mathcal{H},\mathcal{E},B,T)$ where:
\begin{myitemize}
\item $\mathcal{H}$ is the hypothesis space formed of datalog programs
\item $\mathcal{E}$ is the instance space formed of ground atoms
\item $B$ is background knowledge represented as a datalog program
\item $T$ is set of $k$ build tasks $\{E_1,E_2,\dots,E_k\}$ where each $E_i$ is a pair $(E_i^+,E_i^-)$ where $E_i^+\subseteq \mathcal{E}$ and $E_i^-\subseteq \mathcal{E}$ represent positive and negative examples respectively of a target predicate
\end{myitemize}
\end{definition}

\noindent
Note that in a \name{} input, a build task does not need to have negative examples, i.e. $E_i^-$ may be an empty set.

The \name{} problem is to find a consistent program for each task:

\begin{definition}[\textbf{\name{} problem}]
Given a \name{} input $(\mathcal{H},\mathcal{E},B,T)$, the goal is to return a set of hypotheses $\{H_i  \in \mathcal{H} | (E_i^+,E_i^-) \in T, (H_i \cup B \models E^{+}_i) \land  (H_i \cup B \not\models E^{-}_i)\}$
\end{definition}

\subsection{Meta-interpretive learning}
We solve the \name{} problem using MIL, a form of ILP based on a Prolog meta-interpreter.
For brevity, we omit a formal description of MIL, and refer the reader to the literature for more details \cite{crop:thesis}.
We instead provide an informal overview.
A MIL learner is given as input sets of atoms representing positive and negative examples of a target concept, background knowledge in the form of a logic program, and, crucially, a set of second-order formulas called metarules.
A MIL learner works by trying to construct a proof of the positive examples.
It uses the metarules to guide the proof search.
Metarules can therefore be seen as program templates.
Figure \ref{fig:metarules} shows some commonly used metarules.
Once a proof is found a MIL learner extracts a logic program from the proof and checks that it is inconsistent with the negative examples.
If not, it backtracks to consider alternative proofs.

\begin{figure}[ht]
\centering
\begin{tabular}{|l|l|}
\hline
{\bf Name} & {\bf Metarule}\\ \hline
precon & $P(A,B) \leftarrow Q(A),R(A,B)$\\
postcon & $P(A,B) \leftarrow Q(A,B),R(B)$\\
chain & $P(A,B) \leftarrow Q(A,C),R(C,B)$\\
tailrec & $P(A,B) \leftarrow Q(A,C),P(C,B)$\\
\hline
\end{tabular}
\caption{
Example metarules.
The letters $P$, $Q$, and $R$ denote second-order variables.
The letters $A$, $B$, and $C$ denote first-order variables.
}
\label{fig:metarules}
\end{figure}

\subsection{Sample complexity}
\label{sec:sampcomp}
We claim that playing can improve learning performance.
We support this claim by showing that playing can reduce the size of the MIL hypothesis space which in turn reduces sample complexity \cite{mitchell:mlbook} and expected error.
In MIL the size of the hypothesis space is a function of the metarules, the number of background predicates, and the maximum program size.
We restrict metarules by their body size and literal arity:

\begin{definition}
\label{def:ham}
A metarule is in \M{i}{j} if it has at most $j$ literals in the body and each literal has arity at most $i$.
\end{definition}

\noindent
By restricting the form of metarules we can calculate the size of a MIL hypothesis space:

\begin{proposition}
[\textbf{Hypothesis space \cite{crop:thesis}}]
\label{prop:hspace}
Given $p$ predicate symbols and $m$ metarules in \M{i}{j}, the number of programs expressible with $n$ clauses is $(mp^{j+1})^n$.
\end{proposition}

\noindent
We use this result to show the MIL sample complexity:

\begin{proposition}[\textbf{Sample complexity \cite{crop:thesis}}]
\label{thm:sampcomp}
Given $p$ predicate symbols, $m$ metarules in \M{i}{j}, and a clause bound $n$, MIL has sample complexity $s$ with error $\epsilon$ and confidence $\delta$:
\[s \geq \frac{1}{\epsilon} (n \ln(m) + (j+1)n \ln(p) + \ln\frac{1}{\delta})\]
\end{proposition}

\noindent
Proposition \ref{thm:sampcomp} helps explain our idea of playing.
When playing, a learner creates its own play tasks and saves any solutions to the BK, which increases the number of predicate symbols $p$.
The solutions learned whilst playing may in turn help solve the user-supplied build tasks, i.e. could reduce the size $n$ of the target program.
To reuse the example from the introduction, if trying to learn sorting algorithms, a learner could discover the concepts of \emph{partition} and \emph{merge} when playing which could then help learn \emph{quicksort} and \emph{mergesort} respectively.
In other words, the key idea of playing is to increase the number of predicate symbols $p$ in order to reduce the size $n$ of the target program.
We consider when playing can reduce sample complexity:


\begin{theorem}[\textbf{Playgol improvement}]
\label{thm:comp}
Given $p$ predicate symbols and $m$ metarules in \M{i}{j}, let $n$ be the minimum numbers of clauses needed to express a target theory with standard MIL.
Let $n-k$ be the minimum number of clauses needed to express a target theory with Playgol using an additional $c$ predicate symbols.
Let $s$ and $s'$ be the bounds on the number of training examples required to achieve error less than $\epsilon$ with probability at least $1-\delta$ with standard MIL and Playgol respectively.
Then $s > s'$ when:
\[n \ln(p) > (n-k) \ln (p+c)\]
\end{theorem}
\begin{proof}
Follows from Proposition \ref{thm:sampcomp} and rearranging of terms.
\end{proof}

\noindent
Theorem \ref{thm:comp} shows when playing can reduce sample complexity compared to not playing.
In such cases, if the number of training examples is fixed for both approaches, the corresponding discrepancy in sample complexity is balanced by an increase in predictive error \cite{blumer:bound}.
In other words, Theorem \ref{thm:comp} shows that adding extra (sometimes irrelevant) predicates to BK can improve learning performance so long as some can be reused to learn new programs.

\section{\name{}}
\label{sec:playgol}
Algorithm \ref{alg:playgol} shows the \name{} algorithm, which uses Metagol \cite{metagol}, a MIL implementation, as the main learning algorithm.
\name{} takes as input an instance space (the set of all possible examples) $\mathcal{E}$, initial background knowledge BK, a set of user-supplied \emph{build} tasks T$_b$, a playtime value (playtime), and a maximum search depth max$_d$.


\name{} first enters the unsupervised \emph{playing} stage.
In this stage, \name{} creates its own \emph{play} tasks T$_p$ by sampling uniformly with replacement $playtime$ elements from the instance space.
We consider this step to be unsupervised because (1) the tasks are not selected by the user, and (2) no labels (i.e. positive or negative) are provided by the user.
After creating play tasks, \name{} then uses a dependent learning approach \cite{mugg:metabias} to expand the BK.
Starting at depth $d$$=$$1$, \name{} tries to solve each play task using at most $d$ clauses.
To solve an individual task, \name{} calls Metagol.
Each time a play task is solved, the solution (program) is added the BK and can be reused to help solve other play tasks.
Once \name{} has tried to solve all play tasks at depth $d$, it increases the depth and tries to solve the remaining play tasks.
\name{} repeats this process until it reaches the maximum depth (max$_d$), then it returns the initial BK augmented with solutions to the play tasks.

\name{} then enters the supervised \emph{building} stage.
In this stage \name{} tries to solve each user-supplied build task using the augmented BK using a standard independent learning approach, eventually returning a set of induced programs.

\begin{algorithm}[ht]
\begin{myalgorithm}[]
func playgol($\mathcal{E}$,BK,T$_b$,playtime,max$_d$)
    BK = play($\mathcal{E}$,BK,playtime,max$_d$)
    return build(T$_b$,BK,max$_d$)

func play($\mathcal{E}$,BK,playtime,max$_d$)
    T$_p$ = sample($\mathcal{E}$,playtime)
    for d=1 to max$_d$
        for E$^+$ in T$_p$:
            prog = metagol(BK,E$^+$,$\{\}$,max$_d$)
            if prog != null
                BK = BK $\cup$ {prog}
                T$_p$ = T$_p$ $\setminus$ E$^+$
    return BK

func build(T$_b$,BK,max$_d$)
    P = {}
    for (E$^+$,E$^-$) in T$_b$
        prog = metagol(BK,E$^+$,E$^-$,max$_d$)
        if prog != null
            P = P $\cup$ {prog}
    return P
\end{myalgorithm}
\caption{\name{}}
\label{alg:playgol}
\end{algorithm}

\section{Experiments}
\label{sec:experiments}

We claim that playing can improve learning performance.
We now experimentally test our claim.
We test the null hypothesis:

\begin{description}
    \item[Null hypothesis 1] Playing cannot improve learning performance
\end{description}

\noindent
Theorem \ref{thm:comp} shows that playing can reduce sample complexity compared to not playing.
Theorem \ref{thm:comp} does do not, however, state how many play tasks are needed to improve learning performance.
\name{} creates its own play tasks by sampling from the instance space.
Suppose we sampled uniformly at random without replacement from a finite instance space.
Then if we sample enough times we will sample every instance.
One could therefore argue that \name{} is doing nothing more than sampling play tasks that it will eventually have to solve (i.e. \name{} is sampling build tasks whilst playing).
To refute this argument we test the null hypothesis:

\begin{description}
    \item[Null hypothesis 2] Playing cannot improve learning performance without many play tasks
\end{description}

\noindent
To test null hypotheses 1 and 2 we compare \name{}'s performance when varying the number of play tasks.
When there are no play tasks \name{} is equivalent to Metagol.

A key motivation for using MIL is that it supports predicate invention.
Although we provide no theoretical justification, we claim that predicate invention is useful when playing because it allows for problems to be decomposed into smaller reusable sub-problems.
We test this claim with the null hypothesis:

\begin{description}
    \item[Null hypothesis 3] Saving invented predicates whilst playing cannot improve learning performance
\end{description}

\noindent
To test null hypothesis 3 we use a variant of \name{} called \nopi{}.
The only difference between \name{} and \nopi{} is that \name{} uses all the top-level and invented predicates discovered whilst playing when building.
By contrast, \nopi{} uses only the top-level predicates discovered whilst playing when building.
For instance, suppose that whilst playing both \name{} and \nopi{} discovered \emph{mergesort} and did so by inventing predicates for the sub-definitions \emph{split} and \emph{merge}.
Then when building \name{} would use \emph{mergesort}, \emph{split}, and \emph{merge}, whereas \nopi{} would only use \emph{mergesort}.

\subsection{Robot planning}

Our first experiment focuses on learning robot plans.

\paragraph{Materials}
There is a robot and a ball in an $n^2$ space.
The robot can move around and can grab and drop the ball.
The goal is to learn a program to move from the initial state to the final state.
The robot can perform six dyadic actions to transform the state: \tw{up}, \tw{down}, \tw{right}, \tw{left}, \tw{grab}, and \tw{drop}.
Training examples are atoms of the form $f(s_1,s_2)$, where $f$ is the target predicate and $s_1$ and $s_2$ are initial and final states respectively.
We allow \name{} to learn programs using the \emph{ident} and \emph{chain} metarules (Figure \ref{fig:metarules}).
We use $5^2$ and $6^2$ spaces with instance spaces $X_5$ and $X_6$ respectively.
The instance spaces contain all possible $f(s_1,s_2)$ atoms.
The cardinalities of $X_5$ and $X_6$ are approximately $5^8$ and $6^8$ respectively\footnote{In each state there are $n^2$ positions for the robot, $n^2$ positions for the ball, and the robot can or cannot be hold the ball, thus there are approximately $2n^4$ states.
The instance space contains all possible start/end state pairs, thus approximately $2n^8$ atoms}.



\paragraph{Method}
Our experimental method is as follows.
We sample uniformly with replacement 1000 atoms from $X_n$ to form the build tasks $T_b$.
Then for each $p$ in $\{0,200,400,\dots,2000\}$, we call playgol($X_n$,BK,$T_b$,$p$,5) which returns a set of programs $P_p$.
We measure the percentage of correct solutions in $P_p$.
We enforce a timeout of 60 seconds per play and build task.
We measure the standard error of the mean over 10 repetitions.

\paragraph{Results}

Figure \ref{fig:robot-results} shows that \name{} solves more build tasks given more play tasks.
For the $5^2$ space, \name{} solves only 12\% of the build tasks without playing.
The baseline represents the performance of Metagol (i.e. learning \emph{without} play).
By contrast, playing improves performance in all cases.
After 1000 play tasks, \name{} solves almost 100\% of the build tasks.
For the $6^2$ space, the results are similar, where the build performance is only 7\% without playing but over 60\% after 1200 play tasks.
These results suggest that we can reject null hypothesis 1, i.e. we can conclude that playing can improve learning performance.

As already mentioned, one may argue that \name{} is simply sampling build tasks as play tasks.
Such duplication may occur.
In this experiment, for us to sample all of the build tasks we would expect to sample $\Theta(|X_n| \log(|X_n|))$ play tasks\footnote{This problem is an instance of the coupon collectors problem \cite{wiki:ccp}}, which corresponds to sampling approximately 5 million and 24 million tasks for the $5^2$ and $6^2$ spaces respectively.
However, our experimental results show that to solve almost all of the build tasks we only need to sample approximately 1000 and 2000 play tasks for the $5^2$ and $6^2$ spaces respectively.
These values are less than 1/1000 of the expected rate.
Therefore, our experimental results suggest that we can reject null hypothesis 2, i.e. we can conclude that playing can improve learning performance without needing to sample many play tasks.

Finally, Figure \ref{fig:robot-results} shows that \name{} solves more tasks than \nopi{}, although in the $5^2$ space both approaches converge after 2000 play tasks.
A McNemar's test on the results of \name{} and \nopi{} confirmed the significance at the $p < 0.001$ level for the $5^2$ and $6^2$ spaces.
This result suggests that we can reject null hypothesis 3, i.e. we can conclude that predicate invention can improve learning performance when playing.


\begin{figure}[ht]
\centering
\begin{subfigure}{.5\linewidth}
\centering
\begin{tikzpicture}[scale=.55]
    \begin{axis}[
    xlabel=\# play tasks,
    ylabel=\% build tasks solved,
    xmin=0,xmax=2000,
    ymin=0,ymax=100,
    ylabel style={yshift=-4mm},
    legend style={at={(0.6,0.3)},anchor=west,font=\small,nodes={right}}
    ]
\addplot+[error bars/.cd,y dir=both,y explicit] coordinates {
(0,11.91) +- (0,0.28)
(200,35.94) +- (0,3.41)
(400,60.7) +- (0,4.69)
(600,77.99) +- (0,2.41)
(800,91.39) +- (0,1.15)
(1000,97.61) +- (0,0.5)
(1200,98.61) +- (0,0.64)
(1400,99.65) +- (0,0.07)
(1600,99.76) +- (0,0.06)
(1800,99.71) +- (0,0.12)
(2000,99.75) +- (0,0.09)
};

\addplot+[error bars/.cd,y dir=both,y explicit] coordinates {
(0,11.91) +- (0,0.28)
(200,27.76) +- (0,2.51)
(400,37.57) +- (0,3.29)
(600,57.32) +- (0,3.87)
(800,76.16) +- (0,2.15)
(1000,80.56) +- (0,1.66)
(1200,88.7) +- (0,2.35)
(1400,90.48) +- (0,2.03)
(1600,97.3) +- (0,1.06)
(1800,98.48) +- (0,0.5)
(2000,99.12) +- (0,0.28)
};

\addplot+[black,mark=none,dashed] coordinates {
    (0,11.91) +- (0,0.28)
    (2000,11.91) +- (0,0.28)
};
    \legend{\name{},\nopi{},baseline}
    \end{axis}
  \end{tikzpicture}
\caption{$5^2$ space}
\label{fig:robo5res}
\end{subfigure}%
\begin{subfigure}{.5\linewidth}
\centering
\begin{tikzpicture}[scale=.55]
    \begin{axis}[
    xlabel=\# play tasks,
    ylabel=\% build tasks solved,
    xmin=0,xmax=2000,
    ymin=0,ymax=100,
    ylabel style={yshift=-4mm},
    legend style={legend pos=north west,font=\small,style={nodes={right}}}
    ]
\addplot+[error bars/.cd,y fixed,y dir=both,y explicit] coordinates {
(0,6.93) +- (0,0.18)
(200,14.81) +- (0,1.76)
(400,20.93) +- (0,4.12)
(600,38.46) +- (0,4.11)
(800,45.71) +- (0,5.81)
(1000,58.14) +- (0,3.26)
(1200,76.23) +- (0,6.12)
(1400,76.27) +- (0,4.28)
(1600,84.07) +- (0,4.01)
(1800,88.49) +- (0,2.72)
(2000,93.76) +- (0,1.74)
};

\addplot+[error bars/.cd,y fixed,y dir=both,y explicit] coordinates {
(0,6.93) +- (0,0.18)
(200,12.0) +- (0,1.05)
(400,14.48) +- (0,1.06)
(600,20.41) +- (0,2.2)
(800,30.7) +- (0,3.36)
(1000,33.58) +- (0,2.65)
(1200,42.59) +- (0,4.66)
(1400,41.02) +- (0,3.17)
(1600,50.17) +- (0,3.19)
(1800,60.71) +- (0,3.48)
(2000,59.9) +- (0,2.49)
};

\addplot+[black,mark=none,dashed] coordinates {
    (0,7.1) +- (0,0.18)
    (2000,7.1) +- (0,0.18)
};
    \legend{\name{},\nopi{},baseline}
    \end{axis}
  \end{tikzpicture}
\caption{$6^2$ space}
\label{fig:robo6res}
\end{subfigure}
\caption{Robot experiment results. The baseline represents learning without play (i.e. Metagol).}
\label{fig:robot-results}
\end{figure}
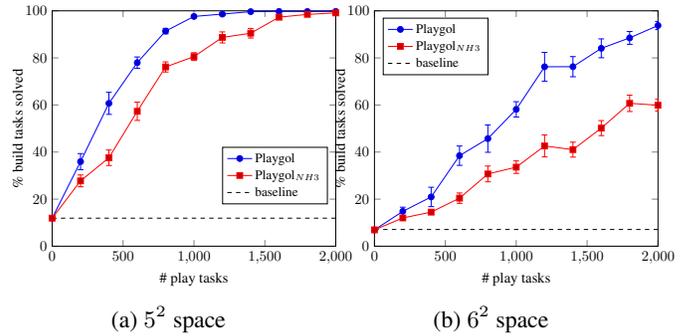

\subsection{String transformations}

Our first experiment tested the null hypotheses in a controlled experimental setting.
We now see whether playing can improve learning performance on `real-world' string transformations.

\paragraph{Materials}
We use 94 real-word string transformation tasks.
Our dataset is based on the dataset from \cite{mugg:metabias}, which in turn is based on \cite{gulwani:flashfill}.
We augmented the dataset with more manually created tasks, taken from a variety of sources (such as stackoverflow, excel forums, etc).
Each task has 10 examples.
Each example is an atom of the form $f(x,y)$ where $f$ is the task name and $x$ and $y$ are input and output strings respectively.
Figure \ref{fig:string-prob} shows three examples for the build task \tw{build\_95}, where the goal is to learn a program that extracts the first three letters of the month name and makes them uppercase.

\begin{figure}[ht]
\centering
\begin{tabular}{l|l}
\textbf{Input}                        & \textbf{Output} \\ \hline
22 July,1983 (35 years old) & JUL\\
30 October,1955 (63 years old) & OCT\\
2 November,1954 (64 years old) & NOV
\end{tabular}
\caption{Examples for the \tw{build\_95} string transformation problem.}
\label{fig:string-prob}
\end{figure}

\noindent
In the build stage we use the real-word tasks.
In the play stage, \name{} samples play tasks from the instance space $X$ formed of random string transformations.
The play tasks are formed from an alphabet with 80 symbols, including the letters a-z, A-Z, the numbers 0-9, and punctuation symbols ($<$,$>$,+,-,\_,etc).
To generate a play task we use the following procedure.
We select a random integer $l$ between 3 and 20 to represent the input length.
We generate a random string $x$ of length $l$ to represent the input string.
We select a random integer $p$ between 3 and 20 and enumerate all programs $P$ of length $p$ consistent with $x$.
We select a random program from $P$ and apply it to $x$ to generate the output string $y$ to form the example $f(x,y)$ where $f$ is the play task name.
This procedure only generates play tasks that are theoretically solvable, i.e. for which there is a hypothesis in the hypothesis space.
Figure \ref{fig:string-play-tasks} shows example play tasks.

\begin{figure}[ht]
\centering
\begin{tabular}{l|l|l}
\textbf{Task} & \textbf{Input} & \textbf{Output} \\ \hline
\tw{play\_9} & .f\textbackslash 73\textbackslash R) & F\\
\tw{play\_52} & @B4\textbackslash X>3MjKdyZzC & B\\
\tw{play\_136} & 9pfy"ktfbS1v& 99PF\\
\tw{play\_228} & I6zihQk- & Q
\end{tabular}
\caption{Examples of randomly generated play tasks for the string transformation experiment.}
\label{fig:string-play-tasks}
\end{figure}

\noindent
The play instance space $X$ contains all possible string transformations consistent with the aforementioned procedure.
The space contains approximately $80^{40}$ atoms\footnote{In the case that the input is length 20 there are $80^{20}$ possible strings, thus $80^{40}$ pairs.}.

We provide \name{} with the metarules \emph{precon}, \emph{postcon}, \emph{chain}, and \emph{tailrec}; the monadic predicates: \tw{empty}, \tw{space}, \tw{letter}, \tw{number}, \tw{uppercase}, \tw{lowercase}; the negations of the monadics \tw{not\_empty}, \tw{not\_space}, etc; and the dyadic predicates \tw{copy}, \tw{skip}, \tw{mk\_uppercase}, \tw{mk\_lowercase}.
These predicates are based on those used in \cite{mugg:metabias}.

\paragraph{Method}
Our experimental method is as follows.
For each real-word string transformation task $t_i$, we sample uniformly without replacement 5 atoms from $t_i$ to form the training examples $t_{i,train}$ and use the remaining 5 atoms as the testing examples $t_{i,test}$.
The set of build tasks $T_b$ is thus the set of individual tasks, e.g. $T_b = \{t_{1,train}, t_{2,train}, \dots\}$.
The set of testing examples $T_t$ is likewise $T_t = \{t_{1,test}, t_{2,test}, \dots\}$.
For each $p$ in $\{0,200,400,\dots,2000\}$, we call playgol($X$,BK,$T_b$,$p$,5) which returns a set of programs $P_p$.
We measure the predictive accuracy of $P_p$ against the testing examples $T_t$.
We enforce a learning timeout of 60 seconds per play and build task.
If \name{} learns no program then every test example is deemed false.
We measure the standard error of the mean over 10 repetitions.

\paragraph{Results}
Figure \ref{fig:string-results} shows the mean predictive accuracies of \name{} when varying the number of play tasks.
Note that we are not interested in the absolute predictive accuracies of \name{}, which are low because of the small timeout and the difficulty of the problems.
We are instead interested in how the accuracies change given more play tasks, and the difference in accuracies between \name{} and \nopi{}.
Figure \ref{fig:string-results} shows that \name{}'s predictive accuracy improves given more play tasks.
\name{}'s accuracy is 25\% without playing.
By contrast, playing improves accuracy in all cases.
After 2000 play tasks, the accuracy is almost 37\%, an improvement of 12\%.

Figure \ref{fig:string-progs} shows an example of when playing improved building performance, where the solution to the build task \tw{b95} is composed of the solutions to many play tasks.
The solutions to the play tasks are themselves are often composed of solutions to other play tasks, including reusing many invented predicates.
This example clearly demonstrates the use of predicate invention to discover highly reusable concepts that build on each other.

Overall the results from this experiment add further evidence for rejecting all the null hypotheses.

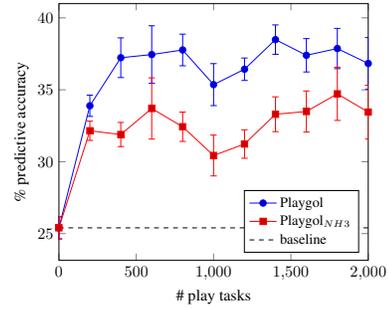
\begin{figure}[ht]
\centering
\begin{tikzpicture}[scale=.6]
    \begin{axis}[
    xlabel=\# play tasks,
    ylabel=\% predictive accuracy,
    xmin=0,xmax=2000,
    ylabel style={yshift=-4mm},
    legend style={legend pos=south east,font=\small,style={nodes={right}}}
    ]
\addplot+[error bars/.cd,y dir=both,y explicit] coordinates {
(0,25.4) +- (0,0.77)
(200,33.89) +- (0,0.74)
(400,37.23) +- (0,1.38)
(600,37.45) +- (0,2.0)
(800,37.77) +- (0,1.1)
(1000,35.36) +- (0,1.46)
(1200,36.43) +- (0,0.78)
(1400,38.49) +- (0,1.02)
(1600,37.4) +- (0,1.16)
(1800,37.87) +- (0,1.4)
(2000,36.83) +- (0,1.8)
};

\addplot+[error bars/.cd,y dir=both,y explicit] coordinates {
(0,25.4) +- (0,0.77)
(200,32.15) +- (0,0.67)
(400,31.89) +- (0,0.84)
(600,33.7) +- (0,2.13)
(800,32.43) +- (0,1.02)
(1000,30.43) +- (0,1.42)
(1200,31.23) +- (0,0.99)
(1400,33.3) +- (0,1.21)
(1600,33.49) +- (0,1.43)
(1800,34.72) +- (0,1.85)
(2000,33.45) +- (0,1.87)
};

\addplot+[
    black,mark=none,dashed
] coordinates {
    (0,25.4) +- (0,0.77)
    (2000,25.4) +- (0,0.77)
};
    \legend{\name{},\nopi{},baseline}
    \end{axis}
  \end{tikzpicture}
  \caption{
  String experiment results.
  }
\label{fig:string-results}
\end{figure}

\begin{figure}[ht]
\centering
\begin{minipage}{.9\linewidth}
\begin{minted}[frame=single,fontsize=\footnotesize]{prolog}
build_95(A,B):-play_228(A,C),play_136_1(C,B).
play_228(A,B):-play_52(A,B),uppercase(B).
play_228(A,B):-skip(A,C),play_228(C,B).
play_136_1(A,B):-play_9(A,C),mk_uppercase(C,B).
play_9(A,B):-skip(A,C),mk_uppercase(C,B).
play_52(A,B):-skip(A,C),copy(C,B).
\end{minted}
\end{minipage}
\caption{
Program learned by \name{} for the build task \tw{build\_95} (Figure \ref{fig:string-prob}).
The solution for \tw{build\_95} reuses the solution to the play task \tw{play\_228} and the sub-program \tw{play\_136\_1} from the play task \tw{play\_136}, where \tw{play\_136\_1} is invented.
The predicate \tw{play\_228} is a recursive definition that corresponds to the concept of ``skip to the first uppercase letter and then copy the letter to the output''.
The predicate \tw{play\_228} reuses the solution for another play task \tw{play\_52}.
Figure \ref{fig:string-play-tasks} shows these play tasks.
}
\label{fig:string-progs}
\end{figure}

\section{Conclusions and future work}
We have introduced the idea of \emph{learning programs through play}.
In this approach, a program induction system creates its own play tasks to solve, tries to solve them, and saves any solutions to its BK, which can then be reused to solve the user-supplied build tasks.
We claimed that playing can improve learning performance.
Our theoretical results support this claim and show that playing can reduce the sample complexity of a learner (Theorem \ref{thm:comp}).
We have implemented our idea in \name{}, a new ILP system.
Our experimental results on two domains (robot planning and string transformations) further support our claim and show that playing can substantially improve learning performance without the need for many play tasks.
Our experimental results also show that predicate invention can improve learning performance because it allows a learner to discover highly reusable sub-programs.

\subsection*{Limitations and future work}

\paragraph{Which domains?}
We have demonstrated the benefits of playing on robot planning and string transformations.
However, the generality of the approach is unclear.
Theorem \ref{thm:comp} shows conditions for when playing can reduce sample complexity and helps explain our empirical results.
Theorem \ref{thm:comp} does not, however, identify a priori on which domains playing is useful.
Our preliminary work suggests that playing is useful in other domains, including to induce graphics programs where playing allows a learner to discover general concepts such as a vertical or horizontal line.
Future work should determine in which domains playing is useful.

\paragraph{How many play tasks?}
Our robot experiments show that as the instance space grows \name{} needs to sample more tasks to achieve high performance.
In future work we want to develop a theory that predicts how many play tasks \name{} needs to substantially improve learning performance.
In addition, we assume a suitably large instance space.
We do not know whether the approach would work when such a space is unavailable.

\paragraph{Better sampling}
In the string transformation experiment playing did not continue to improve performance as it did in the robot experiment.
One explanation for this performance plateau is the \emph{relevancy} of sampled play tasks.
In future work, we want to explore methods to sample more useful play tasks.
For instance, rather than create play tasks in an \emph{unsupervised} manner, it may be beneficial to create play tasks similar to build tasks, i.e. in a \emph{semi-supervised} manner.


\paragraph{Summary}
We have shown that playing can substantially improve learning performance.
We think that the idea of playing (or more verbosely \emph{unsupervised bootstrapping} for \emph{supervised} program induction) opens an exciting research area focusing on how program induction systems can discover their own BK without the need for user-supplied tasks.

\bibliographystyle{named}
\bibliography{playgol}
\end{document}